\documentclass{article}

\usepackage{microtype}
\usepackage{graphicx}
\usepackage{subcaption}
\usepackage{booktabs} 
\usepackage{tcolorbox}
\tcbuselibrary{skins}

\usepackage{hyperref}

\usepackage[preprint]{icml2026}

\usepackage{amsmath}
\usepackage{amssymb}
\usepackage{mathtools}
\usepackage{amsthm}

\usepackage[capitalize,noabbrev]{cleveref}

\theoremstyle{plain}
\newtheorem{theorem}{Theorem}[section]
\newtheorem{proposition}[theorem]{Proposition}
\newtheorem{lemma}[theorem]{Lemma}

\theoremstyle{definition}

\newtheorem{assumption}[theorem]{Assumption}
\theoremstyle{remark}
\newtheorem{remark}[theorem]{Remark}

\usepackage[textsize=tiny]{todonotes}

\icmltitlerunning{Submission and Formatting Instructions for ICML 2026}

\begin{document}

\twocolumn[
  \icmltitle{The Multiple Ticket Hypothesis: \\
  Random Sparse Subnetworks Suffice for RLVR}



  \icmlsetsymbol{equal}{*}

  \begin{icmlauthorlist}
    \icmlauthor{Israel Adewuyi}{equal,sch}
    \icmlauthor{Solomon Okibe}{sch}
    \icmlauthor{Vladmir Ivanov}{sch}
  \end{icmlauthorlist}

  \icmlaffiliation{sch}{Innopolis University, Tartastan, Russia}
  \icmlcorrespondingauthor{Israel Adewuyi}{i.adewuyi@innopolis.university}

  \icmlkeywords{Machine Learning, ICML}

  \vskip 0.3in
]



\printAffiliationsAndNotice{}  

\begin{abstract}
The Lottery Ticket Hypothesis demonstrated that sparse subnetworks can match full-model performance, suggesting parameter redundancy. Meanwhile, in Reinforcement Learning with Verifiable Rewards (RLVR), recent work has shown that updates concentrate on a sparse subset of parameters, which further lends evidence to this underlying redundancy. We study the simplest possible way to exploit this redundancy: training only a randomly selected subset of parameters at extreme sparsities. Empirically, we find that training just 1\% of parameters matches or exceeds full-parameter RLVR finetuning across 3 models and 2 task domains. Moreover, different random masks show minimal overlap ($\leq 0.005$ Jaccard similarity) and yet all succeed, suggesting pretrained models contain many viable sparse subnetworks rather than one privileged set. We term this the \emph{Multiple Ticket Hypothesis}. We explain this phenomenon through the implicit per-step KL constraint in RLVR, which restricts updates to a low-dimensional subspace, enabling arbitrary sparse masks to succeed.

\end{abstract}

\section{Introduction}

Reinforcement Learning with Verifiable Rewards (RLVR) has emerged as a powerful technique for post-training large language models (LLMs), enabling strong performance across domains like mathematics and code generation. Recent works revealed an intriguing property of RLVR: despite updating all model parameters during training, the optimization process naturally concentrates changes on a sparse subset of parameters. Mukherjee et al. ~\yrcite{mukherjee2025reinforcementlearningfinetunessmall} showed that RLVR effectively finetunes only 5-30\% of parameters across different tasks, algorithms, and models, a finding validated by Zhu et al.~\yrcite{zhu2025path}. Critically, this sparsity emerges intrinsically from the policy optimization objective itself rather than from explicit regularization (including KL penalty) or sparse constraints.

These observations raise a natural question: if RLVR naturally updates only a small subset of parameters, what happens when we explicitly restrict training to a random sparse subset from the start? The answer depends on the domain. Chen et al.~\yrcite{chen2021elastic} found random sparse masks \emph{fail} on vision tasks and careful iterative pruning was required to find winning tickets. However, Xu \& Zhang~\yrcite{xu2024random} recently showed random masks at 0.001\% density \emph{succeed} for supervised fine-tuning (SFT) on NLP tasks, suggesting the regime matters. Does this extend to RLVR, which has fundamentally different dynamics: policy gradients, on-policy sampling, and implicit KL constraints~\cite{zhu2025path}?

Motivated by the observed intrinsic sparsity of RLVR  updates \cite{mukherjee2025reinforcementlearningfinetunessmall} and the distinct optimization geometry  revealed by Zhu et al. \yrcite{zhu2025path}, we investigate whether random sparse training can succeed in the RLVR regime. In this work, we demonstrate that training a randomly selected subset ($\leq$ 1\%) of parameters at $\geq$ 99\% sparsity matches or exceeds full-parameter RLVR finetuning performance. We validate this finding across three models (Qwen2.5 0.5B Base and Instruct and 1.5B models) and two distinct task domains (mathematical and logical reasoning).More surprisingly, we find that multiple independent random masks succeed - testing 20 different random 1\% masks reveal that they all achieve comparable performance, despite sharing less than 0.5\% of parameters in common (Jaccard similarity $\approx 0.005$).

We term this phenomenon the \textit{Multiple Ticket Hypothesis}: pretrained LLMs contain many sparse subnetworks capable of successful RLVR finetuning, and random sampling at sufficient density reliably finds one. This stands in contrast to the classical Lottery Ticket Hypothesis \cite{frankle2018lottery}, which posits that sparse subnetworks must be carefully identified through iterative pruning. Our findings suggest that for RLVR objective, the lottery has many winning tickets, and any random draw is likely to succeed.

\begin{figure*}[!th]
  \centering
  \begin{subfigure}{0.49\textwidth}
    \centering
    \includegraphics[width=\linewidth]{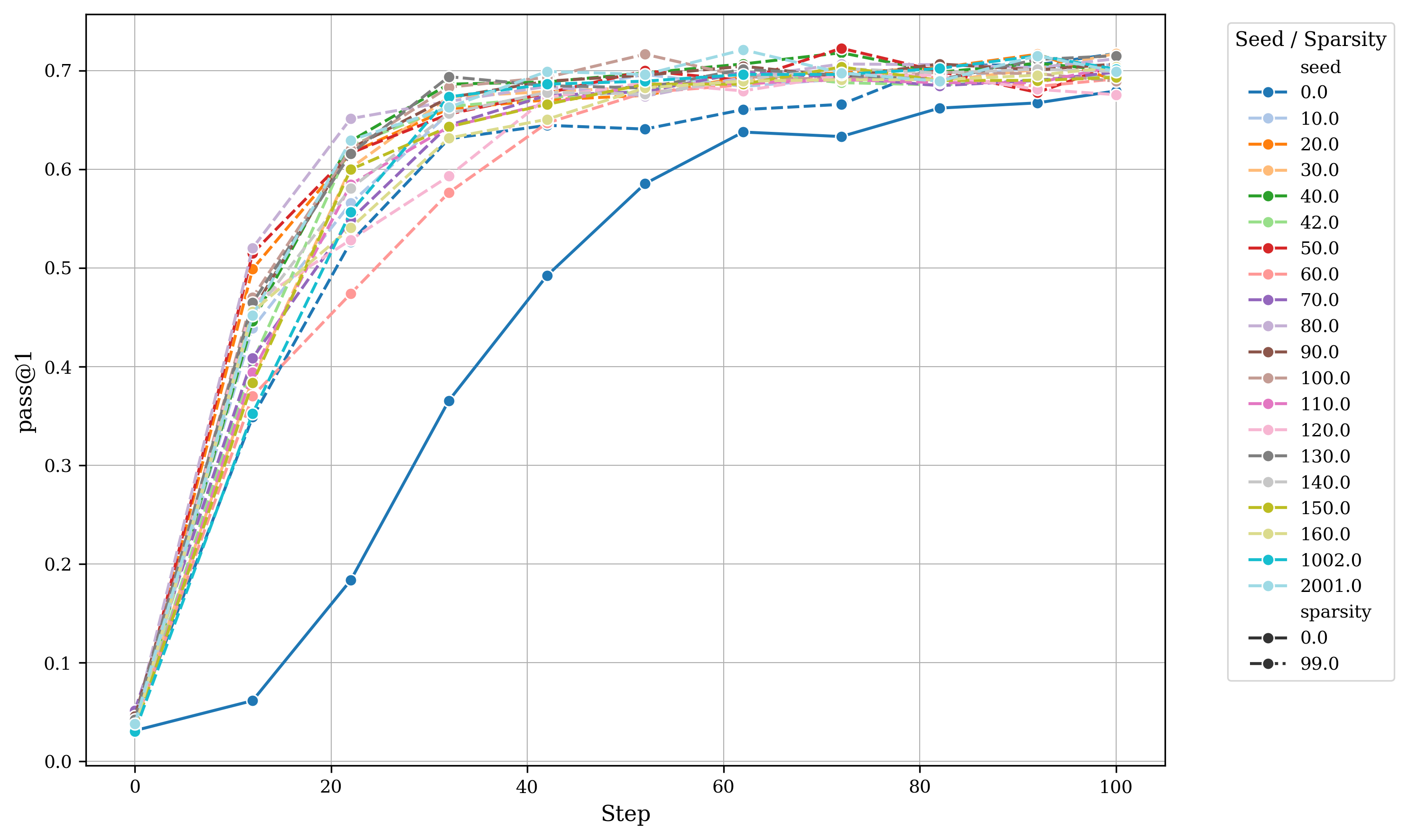}
    \caption{GSM8K (20 seeds)}
  \end{subfigure}\hfill
  \begin{subfigure}{0.49\textwidth}
    \centering
    \includegraphics[width=\linewidth]{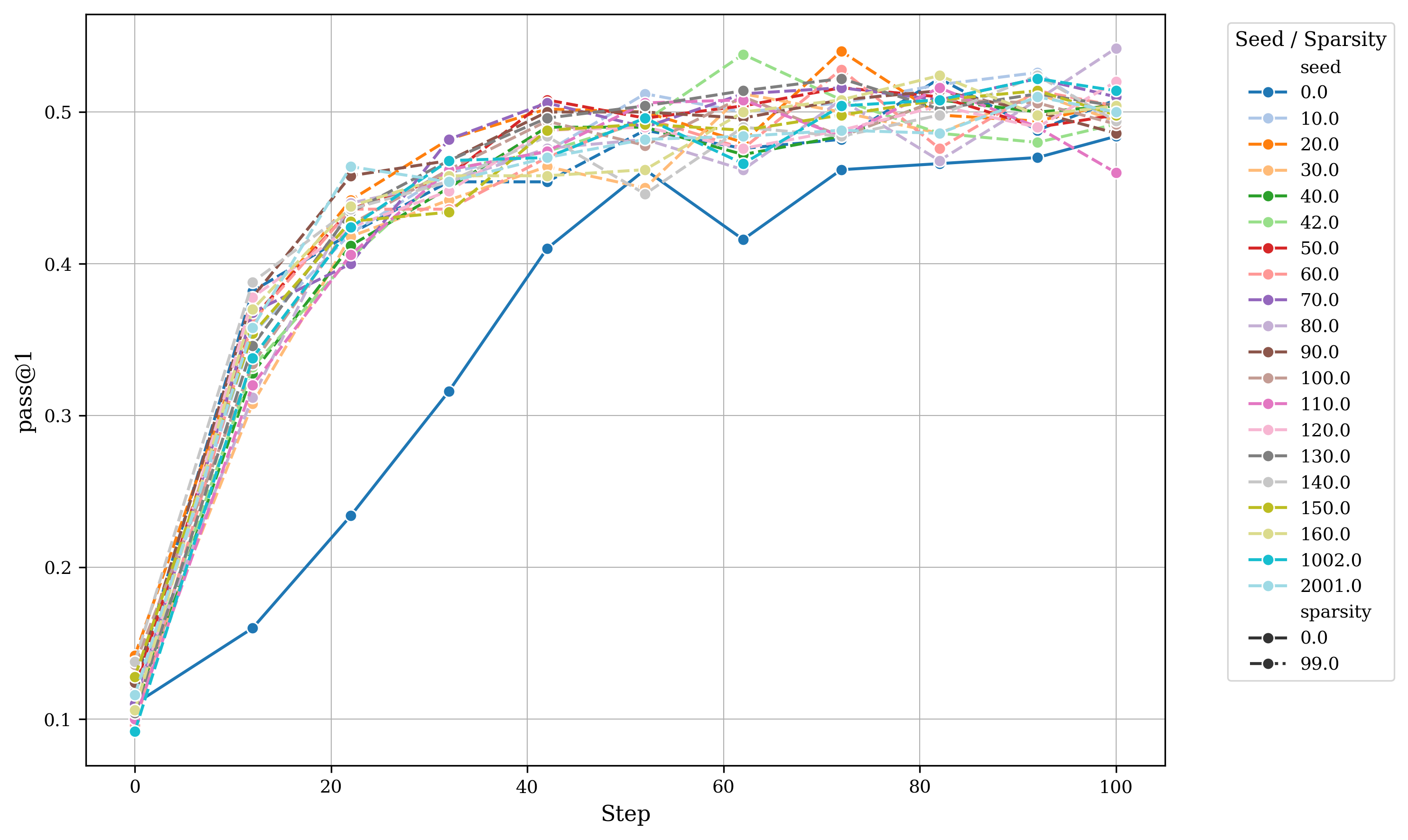}
    \caption{MATH-500 (20 seeds)}
  \end{subfigure}

  \vspace{0.8em}

  \caption{
    \textbf{Multiple random parameter subsets match or exceed full finetuning at 99\% sparsity on Qwen-2.5-1.5B.} 
    Performance of 20 random parameter subsets of Qwen-2.5-1.5B across 100 training steps 
    for GSM8K and MATH-500. 0\% sparsity means full parameter finetuning. 
    99\% sparsity indicates that 1\% of parameters were trained.
  }
  \label{fig:mth_15B_math}
\end{figure*}

To explain this parameter redundancy, we provide a plausible theoretical explanation grounded in the geometry of KL-constrained policy optimization. Building on prior work showing that RLVR updates satisfy implicit KL constraints~\cite{zhu2025path}, we demonstrate that these constraints restrict policy updates to a low-dimensional subspace at each optimization step. This geometric restriction creates vast parameter redundancy: many different parameter subsets can span the same effective update space, making random selection effective. We formalize this intuition through the Fisher information geometry, showing that KL constraints induce a low-rank structure that explains why random sparse training succeeds.

Beyond validating random sparse training as an effective baseline, our work offers practical and conceptual contributions. From a practical standpoint, training only 1\% of parameters provides immediate computational benefits for RLVR research, reducing memory requirements and enabling larger batch sizes. Conceptually, our findings reveal fundamental properties about how RLVR interacts with pretrained representations: the optimization landscape contains extensive flat regions with many viable solutions, suggesting that pretrained models are vastly overparameterized relative to the dimensionality required for RLVR tasks.

Our main contributions are:

\begin{enumerate}
    \item We show that random sparse training at $\geq$99\% sparsity matches full RLVR finetuning across different models and tasks.
    \item We further show that many random subset of parameters succeed and they share minimal parameter overlap, hence the \textbf{Multiple Ticket Hypothesis.}
    \item We provide an explanation for the success of the sparse random subset through the geometry of KL-constrained optimization, showing how trust region constraints create low-dimensional update subspaces that many random parameter subsets can span.
\end{enumerate}

\section{Background and Notation}
\subsection{Reinforcement Learning with Verifiable Rewards}
Following \citet{Guo_2025}, we adopt Group Relative Policy Optimization (GRPO)~\cite{shao2024deepseekmath},
an on-policy reinforcement learning algorithm that extends Proximal Policy Optimization (PPO) \cite{schulman2017proximal} while eliminating the need for a separate value critic.

GRPO estimates advantages using relative rewards within a group of sampled responses. For each prompt~$x$, the current policy~$\pi_\theta$ generates $G$ candidate outputs $\{y_1, \dots, y_G\}$.
Their verifiable rewards $\{R_1, \dots, R_G\}$ are normalized to obtain group-relative advantages
$\hat{A}_i = (R_i - \mu)/\sigma$, where $\mu$ and $\sigma$ are the group mean and standard deviation.

The policy is updated by maximizing a clipped surrogate objective that favors higher-reward responses, regularized by a KL-divergence penalty against a reference policy~$\pi_{\text{ref}}$:

\begin{equation}
\begin{aligned}
\mathcal{L}(\theta) 
&= \mathbb{E}_{x,y_i} \Biggl[ 
   \min\!\Bigl( r_i(\theta)\hat{A}_i,\;
                \text{clip }\!\bigl(r_i(\theta), 1\!-\!\varepsilon, 1\!+\!\varepsilon\bigr)\hat{A}_i \Bigr) \\
&\qquad - \beta\,\text{KL}\bigl(\pi_\theta(\cdot|x) \parallel \pi_{\text{ref}}(\cdot|x)\bigr)
   \Biggr]
\end{aligned}
\end{equation}

with $r_i(\theta) = \pi_\theta(y_i|x) / \pi_{\text{old}}(y_i|x)$ the importance ratio and $\beta$ controlling
regularization strength.

\textbf{Zero-RL Training.} Following \cite{Guo_2025}, 
we perform "zero-RL" training, starting directly from the pretrained 
base model without supervised fine-tuning. We also follow 
\cite{yu2025dapoopensourcellmreinforcement} in setting $\beta=0$ and using token-level policy  gradients which removes explicit KL regularization.

\subsection{Sampling Random Parameters}
To sample a random subset of parameters at x\% sparsity, we iterate through all layers of the model and sample a random subset of the parameters with a fixed seed.

\paragraph{Per-parameter-tensor masking.}
Let the model parameters be organized as a collection of tensors in different layers\footnote{In practice, the parameters are organized into tensors which are organized into layers.}.
\[
\{\theta^{(l)}\}_{l=1}^L.
\]
For each parameter tensor $\theta^{(l)}$, we independently construct a binary mask
\[
m^{(l)} \in \{0,1\}^{\mathrm{shape}(\theta^{(l)})}.
\]
Given a target sparsity level $s \in [0,1)$ with keep ratio $p = 1 - s$, we sample
\[
k^{(l)} = \lfloor p \cdot |\theta^{(l)}| \rfloor
\]
entries of $\theta^{(l)}$ uniformly at random \emph{without replacement} and set the corresponding $k^{(l)}$ entries of $m^{(l)}$ to one, with all remaining entries set to zero.

This procedure results in approximately uniform sparsity across parameter tensors while the total number of active parameters is also approximately at a ratio of $p$.

\subsection{Masked Training}
Masks are sampled once at initialization and held fixed throughout training. During training, gradients are computed densely for all parameters. The effective gradient used for optimization is given by
\[
\nabla_{\theta^{(l)}}^{\text{masked}} \mathcal{L}
=
m^{(l)} \odot \nabla_{\theta^{(l)}} \mathcal{L},
\]
where $\odot$ denotes elementwise multiplication. Parameters corresponding to zero entries in the mask receive zero gradient updates at all training steps and remain fixed at their initialization values. Optimizer states (e.g., momentum terms) are also maintained only for unmasked parameters.

\section{Experimental Setup}

We validate our findings across multiple models and tasks to demonstrate that random sparse training generalizes beyond specific configurations. 

\subsection{Models and Datasets}

\textbf{Models.} We conduct experiments across three models of varying scales: Qwen2.5-0.5B (Base and Instruct) and Qwen2.5-1.5B. \cite{qwen2025qwen25technicalreport}.

\textbf{Tasks.} We evaluate on two distinct domains:

\textit{Mathematical Reasoning.} We train Qwen2.5-1.5B and Qwen2.5-0.5B (Zero RL training \cite{Guo_2025, zeng2025simplerl}) on Hendrycks MATH \cite{hendrycks2021measuringmathematicalproblemsolving} and evaluate on both MATH-500 and GSM8K \cite{cobbe2021trainingverifierssolvemath}, reporting pass@1 accuracy.

\textit{Logical Reasoning.} We train Qwen2.5-0.5B-Instruct and evaluate on Alphabet Sort \footnote{https://huggingface.co/datasets/kalomaze/alphabetic-arxiv-authors-it1}, a multi-turn task that requires sorting an increasing list of names. Training and evaluation are limited to two turns of interaction.

Table~\ref{tab:models-tasks} summarizes the model-task combinations used in our experiments and we refer readers to Appendix ~\ref{sec:complete-exp-setup} for the complete table of hyperparameters.

\begin{table}[t]
  \caption{Summary of models, training tasks, and evaluation tasks.}
  \label{tab:models-tasks}
  \begin{center}
    \begin{small}
      \begin{sc}
        \begin{tabular}{lcc}
        \toprule
          Model & Train set & Eval set \\
          \midrule
          Qwen2.5-0.5B-It & Alphabet Sort & Alphabet Sort  \\
          Qwen2.5-0.5B          & GSM8K & GSM8K \\
          Qwen2.5-1.5B          & MATH & GSM8K, MATH-500 \\
          \bottomrule
        \end{tabular}
      \end{sc}
    \end{small}
  \end{center}
  \vskip -0.1in
\end{table}

\subsection{Training Configuration}

\textbf{Optimization.} Unless otherwise stated, we use the AdamW optimizer \cite{loshchilov2017decoupled}.

\textbf{Implementation.} All experiments use a fork of the Prime-RL library \cite{primeintellect2025prime-rl} with Prime Environments for standardized task interfaces.

\begin{figure*}[!th]
  \centering
  \begin{subfigure}{0.49\textwidth}
    \centering
    \includegraphics[width=\linewidth]{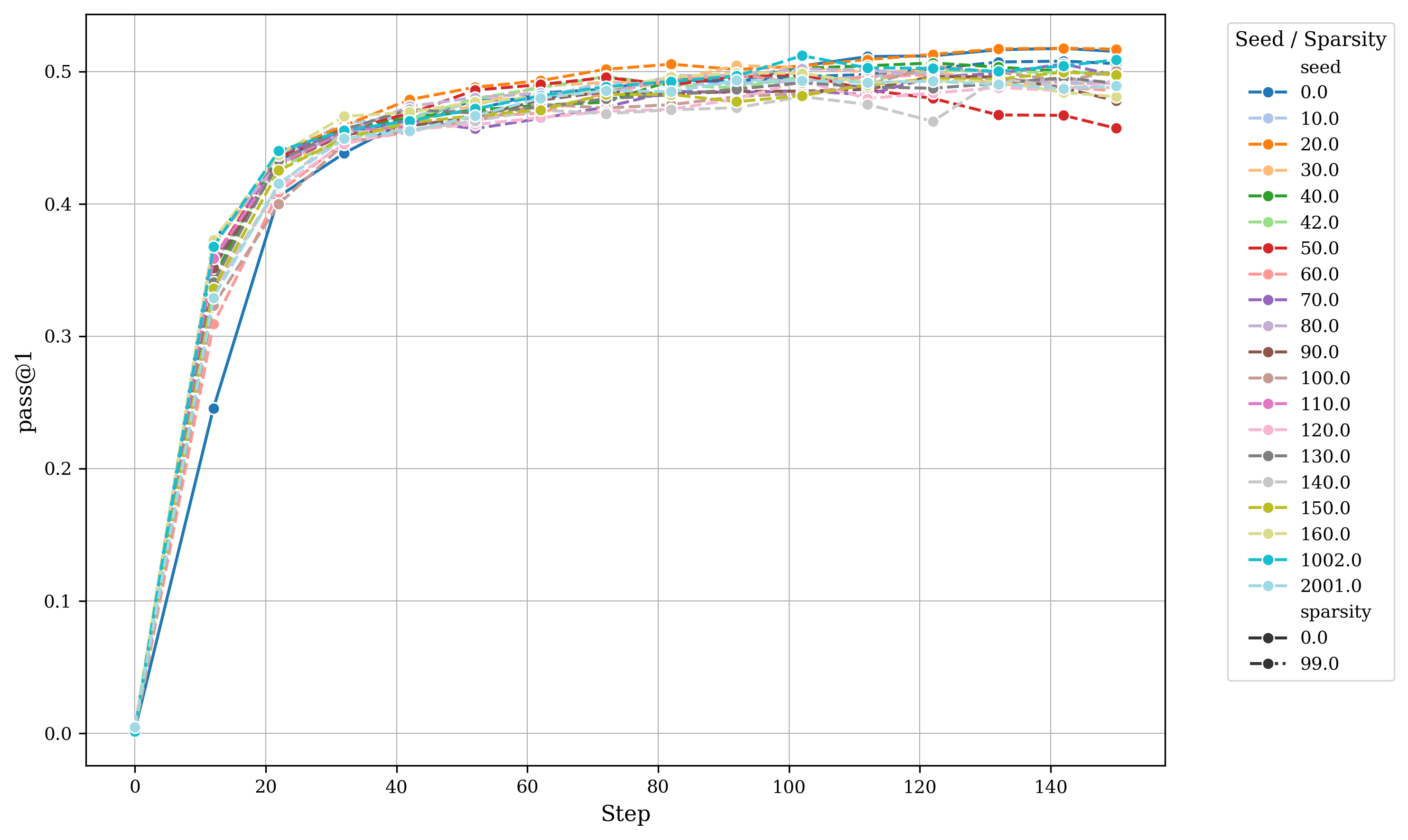}
    \caption{Alphabet sort (20 seeds)}
  \end{subfigure}\hfill
  \begin{subfigure}{0.49\textwidth}
    \centering
    \includegraphics[width=\linewidth]{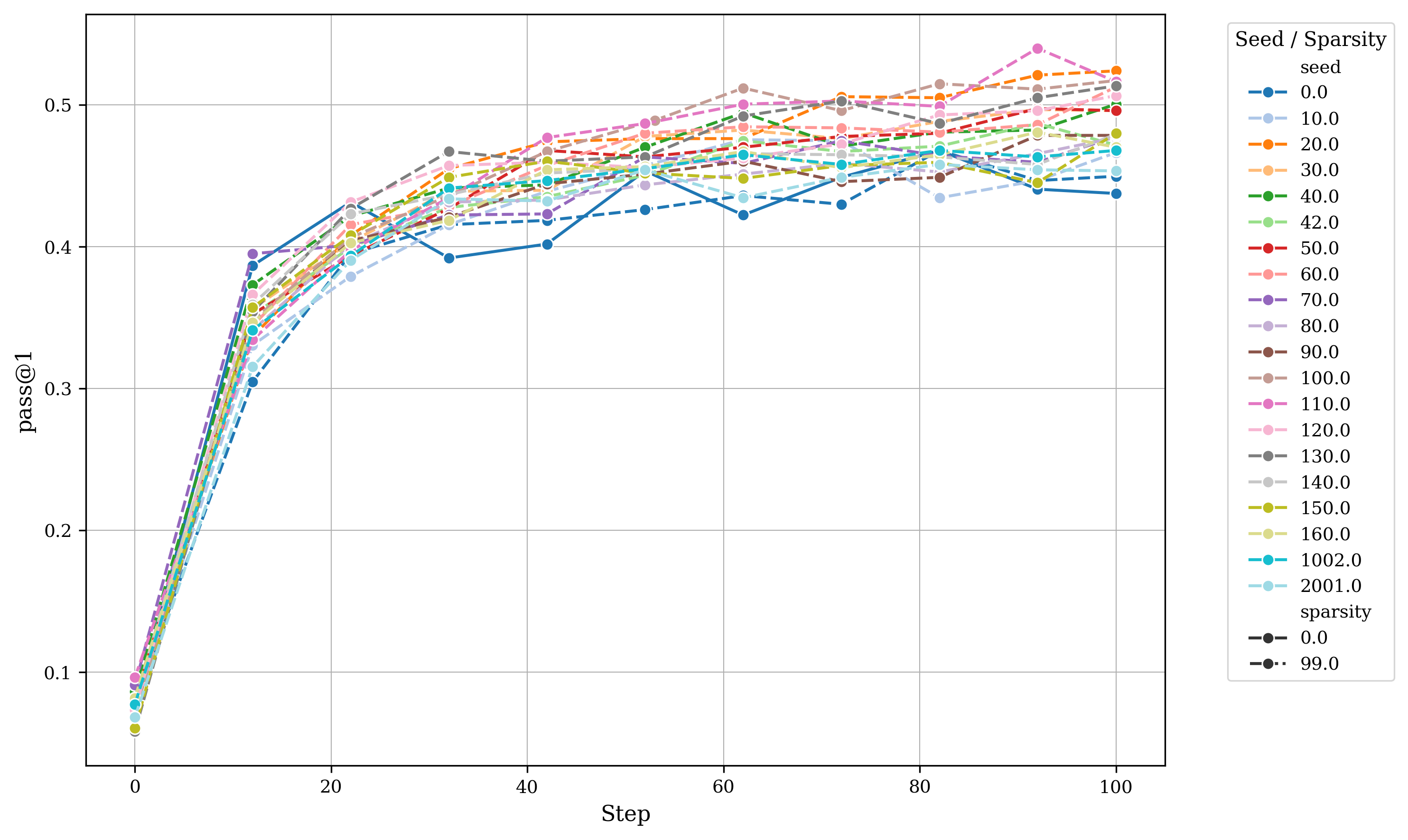}
    \caption{GSM8K (20 seeds)}
  \end{subfigure}

  \vspace{0.8em}

  \caption{
    \textbf{Multiple random parameter subsets match or exceed full finetuning at 99\% sparsity on Qwen-2.5-0.5B.} 
    Performance of 20 random parameter subsets of Qwen-2.5-0.5B across 100 training steps for GSM8K and 150 steps for Alphabet sort.
  }
  \label{fig:mth_05B_math}
\end{figure*}

\subsection{Experimental Design}
\label{sec:experimental-design}

Our experimental design tests whether random sparse training can match full-parameter RLVR finetuning across multiple random initializations.

\textbf{Multi-seed protocol.} For each model-task-sparsity configuration, we train multiple independent models using different random masks, each requiring different learning rates. All other hyperparameters, including the training seed, remain fixed across runs. We report mean performance across the five masks, with error bars indicating standard deviation. This design allows us to assess both the effectiveness of random sparse training and the variance introduced by different random parameter selections.

\textbf{Sparsity levels and learning rates.} Table~\ref{tab:sparsity_active_lr} summarizes the sparsity levels tested, the corresponding number of active parameters, and the learning rates used.

We provide complete hyperparameters and prompt templates in Appendix~\ref{sec:complete-exp-setup}.

\begin{figure*}[t]
  \centering
  \includegraphics[width=\textwidth]{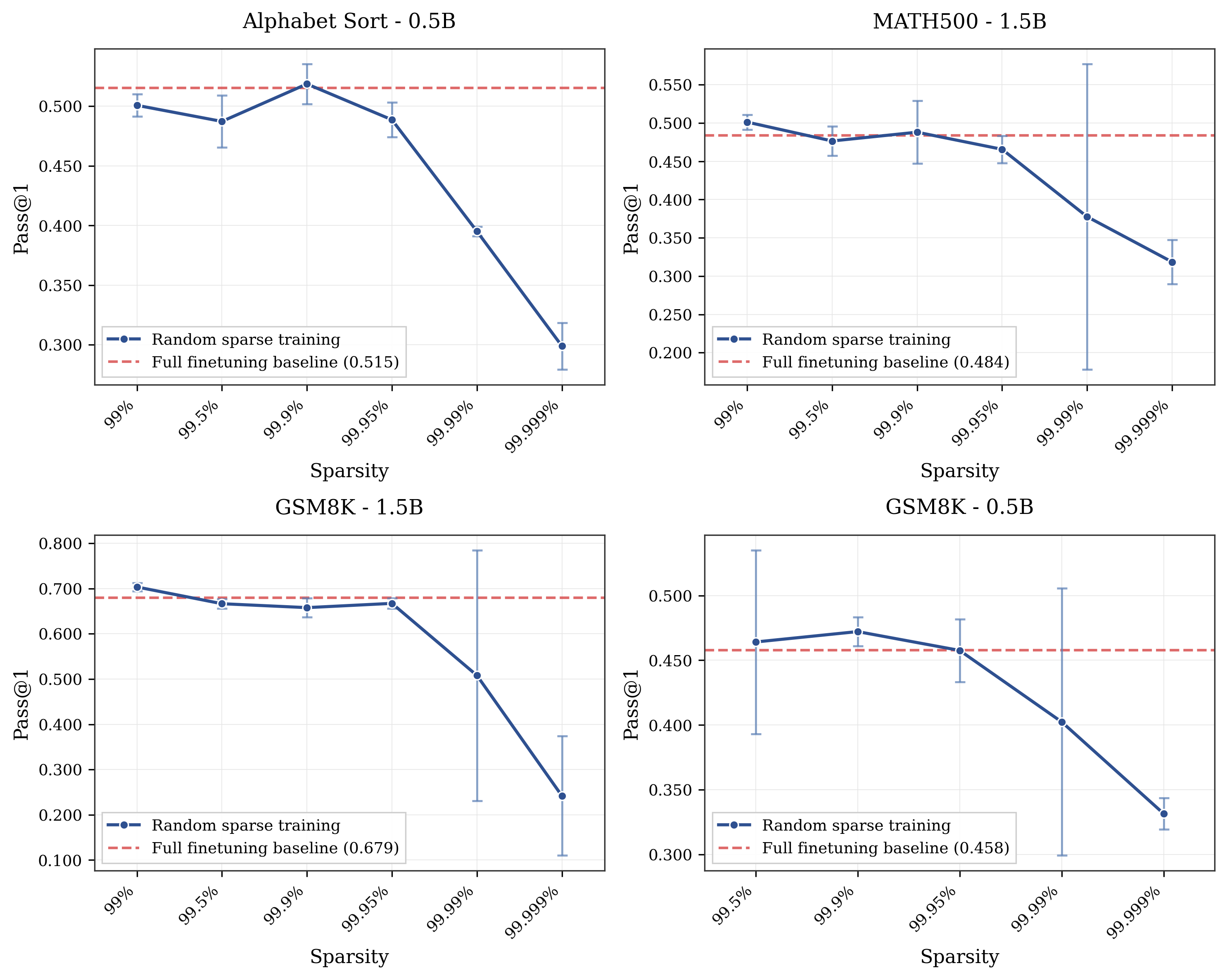}
  \caption{\textbf{Random sparse training matches full finetuning at different sparsities.} Validation performance across sparsity levels for three tasks. Error bars show variation across five random masks. Horizontal dashed lines indicate full-parameter baselines. All results use best learning rate from sweep for each configuration.}
  \label{fig:sparsity-sweep}
\end{figure*}

\begin{table}[h]
  \centering
  \caption{Sparsity levels, active parameter counts, and used learning rates.}
  \label{tab:sparsity_active_lr}
  \footnotesize
  \begin{tabular}{lrrrr}  
    \toprule
    Sparsity (\%) & \multicolumn{2}{c}{Active parameters} & \multicolumn{2}{c}{Learning rate} \\
    \cmidrule(lr){2-3} \cmidrule(lr){4-5}
                  & 0.5B          & 1.5B          & 0.5B     & 1.5B     \\
    \midrule
    0             & $\sim$ 0.5B    & $\sim$ 1.5B    & $1\times10^{-6}$     & $1\times10^{-6}$      \\
    99            & $\sim$ 4.9M    & $\sim$ 15M     & $1\times10^{-4}$     & $1\times10^{-4}$      \\
    99.5          & $\sim$ 2.4M    & $\sim$ 7.7M            & $1\times10^{-4}$      & $1\times10^{-4}$       \\
    99.9          & $\sim$ 490K    & $\sim$ 1.5M    & $1\times10^{-3}$     & $1\times10^{-3}$      \\
    99.95         & $\sim$ 247K    & $\sim$ 770K            & $1\times10^{-3}$     & $1\times10^{-3}$       \\
    99.99         & $\sim$ 49K     & $\sim$ 150K    & $1\times10^{-3}$     & $4\times10^{-3}$      \\
    99.999        & $\sim$ 4K      & $\sim$ 15K     & $5\times10^{-3}$     & $1\times10^{-2}$      \\
    \bottomrule
  \end{tabular}
\end{table}

\section{Results}
\label{sec:results}

We show that pretrained language models contain a large number of disjoint sparse subnetworks that each suffice for effective reinforcement learning with verifiable rewards (RLVR) finetuning. We call this the \emph{Multiple Ticket Hypothesis} (MTH). The evidence comes from training extremely sparse random subnetworks and observing consistent high performance despite negligible parameter overlap.

\subsection{Multiple Ticket Hypothesis: Many Disjoint Subnetworks Succeed}
\label{subsec:multiple-tickets}

At 99\% sparsity (1\% trainable parameters), we trained 20 independent models using independently sampled random masks on Qwen2.5-1.5B (GSM8K and MATH-500) and Qwen2.5-0.5B-Instruct (Alphabet Sort). 

Figures~\ref{fig:mth_15B_math} and \ref{fig:mth_05B_math} show that nearly all masks reach performance equal to or better than full-parameter finetuning. Average Jaccard similarity between any pair of successful masks is $\approx 0.005$ (Table~\ref{tab:jaccard}), exactly as expected for random selection at this density. 

This near-zero overlap rules out the existence of a single privileged subnetwork. Instead, pretrained models appear to contain combinatorially many viable sparse tickets for RLVR and any sufficiently large random draw succeeds.

With 490M (Qwen-2.5-0.5B) parameters and 99\% sparsity, there are theoretically 
$\binom{490M}{4.9M}$ possible masks. Our 20/20 success rate with 
$\leq$ 0.5\% overlap suggests the number of viable masks scales 
combinatorially, which is vastly more than the "one winning ticket" paradigm.

\subsection{Performance Matches Full Finetuning Down to Extreme Sparsity}
\label{subsec:sparsity-sweep}

We next sweep sparsity using 5 random masks per level (seeds 0, 10, 42, 1002, 2001). Figure~\ref{fig:sparsity-sweep} reports mean validation pass@1 $\pm$ std across masks.

Key observations:
\begin{itemize}
    \item \textbf{99\% to 99.95\% sparsity (1\% to 0.05\% parameters)}: Performance at these sparsities match, exceeds or slightly underperforms full finetuning on all tasks and models (GSM8K, MATH-500, Alphabet Sort; 0.5B and 1.5B scales).
    \item \textbf{99.99\% and beyond}: Sharp degradation, with collapse below $\sim$0.01--0.001\% trainable parameters.
\end{itemize}

The consistent transition point across tasks and scales suggests a task-agnostic lower bound on the effective trainable dimensionality required for RLVR, rather than a gradual degradation.

\subsection{Comparison to Structured Sparsity Baselines}
\label{subsec:baselines}

At fixed budget (99\% sparsity) (see Appendix \ref{sec:structured_baselines}, a random mask outperforms structured alternatives (first-layer only, last-layer only) on all the tasks and model combinations we test in this work. No architectural bias or importance scoring is needed; random selection is sufficient.

\subsection{Failure Cases}
\label{subsec:failures}

Model collapse is a well established failure mode in RL training \cite{kumar2024training, dasagi2019ctrl, deng2025grpo, dong2025rlpluscounteringcapabilityboundary}, but we also observe a lot more model collapse with the random masked training and we also attribute the high variance at higher sparsities in Figure \ref{fig:sparsity-sweep} to this. 


\subsection{Summary}
\label{subsec:results-summary}

Our experiments establish three main findings:
\begin{enumerate}
    \item Training a random 1\% of parameters is sufficient to match full RLVR finetuning across model scales and reasoning domains.
    \item Successful subnetworks are highly non-overlapping (Jaccard $\leq 0.005$), supporting the \emph{Multiple Ticket Hypothesis}: pretrained LLMs contain many, likely combinatorially many, viable sparse tickets for RLVR.
    \item Performance depends primarily on the \emph{number} of trainable parameters (effective dimensionality), not their specific identity.
\end{enumerate}

These results point to extreme functional redundancy in the parameter space of pretrained models when optimized under RLVR objectives.


\begin{table}[!th]
\centering
\caption{Jaccard similarity between pairs of successful masks. Values show mean across all mask pairs between the 5 masks for each model. Expected overlap for random masks are also shown.}
\label{tab:jaccard}
\begin{tabular}{lccc}
\toprule
\textbf{Configuration} & \textbf{99\%} & \textbf{99.9\%} & \textbf{99.99\%} \\
\midrule
Qwen2.5-0.5B & $0.005$ & $0.0005$ & $0.000055$ \\
Qwen2.5-1.5B & $0.005031$ & $0.000498$ & $0.0000528$ \\
\midrule
Expected (random) & $0.005$ & $0.0005$ & $0.00005$ \\
\bottomrule
\end{tabular}
\end{table}

\begin{figure}[ht]
  \vskip 0.2in
  \begin{center}
    \centerline{\includegraphics[width=\columnwidth]{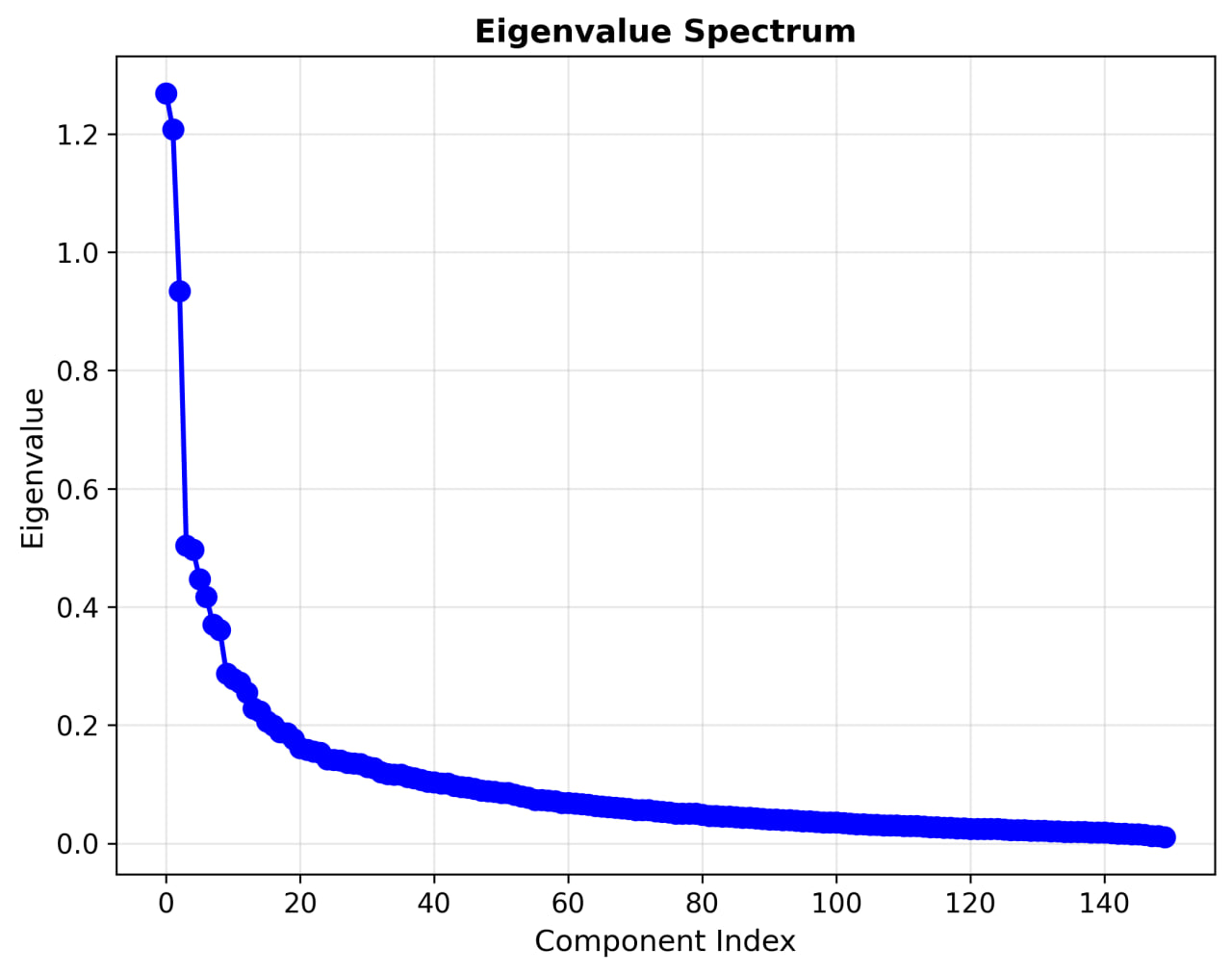}}
    \caption{
      Eigenspectrum analysis of the gradients.
    }
    \label{fig:eigenspectrum}
  \end{center}
\end{figure}

\section{Theoretical Explanation}
\label{sec:theory}

We explain the Multiple Ticket Hypothesis through the geometry of KL-constrained policy optimization. Our framework shows that per-step KL constraints restrict policy updates to a low-dimensional subspace, creating sufficient redundancy for arbitrary sparse masks to succeed.

\paragraph{Setup and Assumptions.}
Let $\pi_\theta(y|x)$ be a policy with parameters $\theta \in \mathbb{R}^d$, and let $F(\theta)$ denote its Fisher information matrix with eigenvalues $\lambda_1 \geq \cdots \geq \lambda_d$ and corresponding orthonormal eigenvectors $v_1, \ldots, v_d$. We assume: (1)~\textbf{Low effective rank}: the top $r \ll d$ eigenvalues capture nearly all Fisher variance, i.e., $\sum_{i=1}^r \lambda_i / \sum_{i=1}^d \lambda_i \geq 1 - \epsilon$ for small $\epsilon$; (2)~\textbf{Delocalized eigenvectors}: $\|v_i\|_\infty \leq \mu/\sqrt{d}$, meaning no single parameter dominates any eigenvector; (3)~\textbf{Small per-step updates}: $\|\Delta\| = O(\sqrt{K})$ where $K$ is the KL bound, ensuring second-order approximations hold.

\begin{proposition}[Low-Dimensional Policy Sensitivity]
\label{prop:lowdim}
Under assumptions (1) and (3), for any update $\Delta$ satisfying $D_{\mathrm{KL}}(\pi_{\theta+\Delta} \| \pi_\theta) \leq K$, the policy change depends only on the projection of $\Delta$ onto the top-$r$ eigenspace $U = \mathrm{span}\{v_1, \ldots, v_r\}$. Components orthogonal to $U$ have negligible impact.
\end{proposition}

\begin{proposition}[Sufficiency of Random Masks]
\label{prop:random}
Under assumptions (1)--(3), let $S \subset \{1, \ldots, d\}$ be a random subset of size $k > r$. With high probability, there exists an update $\Delta_S$ supported on $S$ that approximates any KL-feasible update in the top-$r$ subspace: $\|\Delta_\parallel - \Delta_S\|_F \leq \eta$, where $\|\cdot\|_F$ is the Fisher norm and $\eta \to 0$ as $k$ increases.
\end{proposition}

Complete theoretical outline appear in Appendix~\ref{app:proofs}. The key insight is that KL constraints create a trust region aligned with Fisher eigenvectors. Since only the top $r$ directions matter for policy change (Proposition~\ref{prop:lowdim}) and eigenvectors are delocalized across parameters, random sampling of $k > r$ parameters reliably captures this subspace (Proposition~\ref{prop:random}). This explains both why random masks work and why multiple independent masks succeed despite minimal overlap---they each span the same low-dimensional policy-relevant subspace through different parameter combinations.

\paragraph{Connection to Empirics.}
Our eigenspectrum analysis (Figure~\ref{fig:eigenspectrum}) reveals an effective rank $r \approx 44$ for Qwen2.5-0.5B on Alphabet Sort, representing an intrinsic dimensionality of $\sim$0.0000089\% of 490M parameters. This explains the observed sparsity threshold: performance degrades sharply below $\sim$0.01\% trainable parameters (Figure~\ref{fig:sparsity-sweep}). The combinatorial number of valid masks, $\binom{d}{k}$ choices each spanning the same $r$-dimensional subspace, directly yields the Multiple Ticket Hypothesis.

\begin{remark}
Our experiments use $\beta=0$ (no explicit KL penalty), yet Zhu et al.~\yrcite{zhu2025path} established that on-policy RLVR satisfies implicit KL constraints through clipping and on-policy sampling. Thus $\beta=0$ represents a conservative test: success under weaker implicit constraints implies success under stronger explicit regularization.
\end{remark}

\section{Related Work}

\subsection{Post-Training of Large Language Models and RLVR}
Pretrained large language models \cite{radford2018improving,brown2020language,touvron2023llamaopenefficientfoundation,anil2023palm,achiam2023gpt,chowdhery2023palm, li2024multimodal} require post-training to achieve strong performance on downstream tasks. The main paradigms include supervised fine-tuning (SFT) 
\cite{chung2024scaling, wei2021finetuned, dodge2020fine, howard2018universal} and reinforcement learning (RL), often applied sequentially \cite{ouyang2022training,ziegler2019fine,Guo_2025}.

Recent progress in LLM reasoning ~\cite{Guo_2025} shows that Reinforcement Learning with Verifiable Rewards (RLVR) excels in domains with clear verification and correctness checks, such as mathematics, code generation, and logical reasoning. These methods typically build on policy optimization algorithms such as PPO \cite{ouyang2022training} or the more memory-efficient GRPO \cite{shao2024deepseekmath} and variants that further improve them \cite{yu2025dapoopensourcellmreinforcement, liu2025understanding, liu2025prorl, zheng2025group}.

RLVR's optimization dynamics have drawn attention. Mukherjee et al. \cite{mukherjee2025reinforcementlearningfinetunessmall} showed that despite updating all parameters, RLVR concentrates changes on 5-30\% of them. Zhu et al. \cite{zhu2025path} revealed implicit per-step KL constraints in RLVR even without explicit regularization ($\beta=0$), distinguishing it from SFT. These observations inspired our explicit sparse training experiments, but unlike their focus on analyzing post-hoc subnetworks, we show that random sparse subnetworks at $\geq$99\% sparsity match or exceed full RLVR performance from the start, without prior training or identification.

\subsection{Sparsity in Neural Network Training and Lottery Tickets}
The Lottery Ticket Hypothesis~\cite{frankle2018lottery, malach2020proving} established that dense networks contain sparse subnetworks matching full-model performance, though identifying these ``winning tickets'' requires iterative pruning. Subsequent work extended LTH to various domains, including deep reinforcement learning~\cite{yu2019playing, vischer2021lottery, graesser2022state}, and identified task-specific subnetworks in language models that can mitigate catastrophic forgetting~\cite{panda2024lottery, panigrahi2023task}. Chen et. al ~\yrcite{chen2021elastic} showed that random sparse subnetworks underperformed winning tickets identified by iterative pruning on vision tasks. 


In contrast to LTH's emphasis on a single privileged subnetwork identified via pruning, our Multiple Ticket Hypothesis reveals that pretrained LLMs contain many viable sparse subnetworks for RLVR and sampling a random subset of parameters at sufficient density reliably discovers one.

\subsection{Parameter-Efficient Fine-Tuning (PEFT)}
PEFT methods cut costs for adapting / finetuning large models and numerous studies have established that they operate in constrained parameter subspaces \cite{ben-zaken-etal-2022-bitfit, hu2022lora, albert2025randlora, ansell2024scaling, zhang2023adalora, wu2024reft, li-liang-2021-prefix, pmlr-v97-houlsby19a, malladi2023fine}. 

Of most interest is LoRA \cite{hu2022lora}, which constrains update to low rank matrices i.e it learns the low-rank subspace during training via the adapter module. Schulman, John and Thinking Machines Lab ~\yrcite{schulman2025lora} further showed that rank-1 LoRA matches full RL finetuning. In this work, we sample the subspace, rather than learn it and our focus is on RLVR, while most of the works in this domain have been focused on SFT.


\subsection{Sparsity in RLVR}

Mukherjee et al. \yrcite{mukherjee2025reinforcementlearningfinetunessmall} observed RLVR's intrinsic sparse updates (5-30\%), conjecturing post-hoc subnetworks recover performance. Zhu et al. \yrcite{zhu2025path} linked sparsity to model-conditioned bias, emphasizing off-principal updates and spectral preservation. We align with this geometry but complement it: our framework, rooted in per-step KL constraints inducing low-dimensional subspaces, explains why arbitrary random masks at $>$99\% sparsity succeed without prior training. This shifts emphasis from describing sparsity to leveraging redundancy for efficient RLVR, supporting our Multiple Ticket Hypothesis over singular subnetworks.

\subsection{Random Sparse Training for Fine-Tuning}

Most related, Xu \& Zhang \yrcite{xu2024random} demonstrated random masks at 0.001\% trainable parameters match full SFT on NLP (language understanding and comprehension tasks), attributing flatter landscapes (smaller Hessians) and higher learning rates to overparameterization, analyzed via linear regression. Concurrently, Sampreeth et al. explored using expander graph masks instead of random masks for the initial subnetworks.

We extend this to RLVR by showing that RLVR shows greater redundancy, with many independent masks succeeding. Mechanisms differ—SFT's unconstrained updates vs. RLVR's policy gradients, on-policy sampling, and implicit KL constraints \cite{zhu2025path}—leading to RLVR-specific low-rank Fisher structure from trust regions, not general flat landscapes. We provide empirical multiplicity evidence via Jaccard analysis and test on reasoning tasks with Qwen models, unlike their SFT classification task as RLVR success isn't implied by SFT due to differing dynamics.

\subsection{Fisher Information, Policy Optimization Geometry, and Intrinsic Dimensionality}

Fisher information aids understanding training dynamics, parameter importance for transfer, and generalization. In policy optimization, natural gradient methods use it for stable updates via metric tensors.

We build on this by showing KL constraints in RLVR create low-rank gradient Fisher matrices, restricting updates to low-dimensional subspaces that enable random sparse training. This geometric view explains mask success and multiplicity.

The success ties to intrinsic dimensionality: Aghajanyan et al. \yrcite{aghajanyan2021intrinsic} showed fine-tuning needs few parameters despite billions total, aligning with overparameterization theory \cite{allen2019convergence,du2018gradient}. Our framework applies this to RLVR, where KL induces policy-relevant low-dimensional subspaces, and delocalization lets random masks span them. The Multiple Ticket Hypothesis follows from trust-region methods in overparameterized networks.

\section{Discussion and Conclusion}
\label{sec:discussion}

Our investigation into random sparse training for Reinforcement Learning with Verifiable Rewards (RLVR) reveals a striking property of pretrained language models: the existence of combinatorially many viable subnetworks capable of matching full-parameter performance. This \textit{Multiple Ticket Hypothesis} (MTH) fundamentally shifts our understanding of parameter redundancy in the RLVR regime.

\textbf{RLVR optimization} We showed in Section \ref{sec:theory} that the the Fisher is low-dim in RLVR. We conjecture that RLVR is locally optimizing a flat loss landscape. This intuition is further surported by preliminary experiments from Mukherjee et al., \yrcite{Mukherjee2025AdamSGD} where they showed that even simple optimizers like SDG also match and outperform optimizers such as Adam(W).

These findings establish random sparse training as a strong baseline for parameter-efficient RLVR and suggest new directions for understanding how reinforcement learning interacts with pretrained language model representations.

A potential point of contention is the fact that Mukherjee et al., \yrcite{mukherjee2025reinforcementlearningfinetunessmall} claim that the updates during RLVR are nearly full rank. We empirically show that the gradients are effectively low-rank. This distinction comes from the fact that first, Mukherjee et al., \yrcite{mukherjee2025reinforcementlearningfinetunessmall} estimate the rank from $\Delta = \theta_{final} - \theta_{init}$, while we estimate effective rank from the gradients (Appendix \ref{sec:eigenspectrum-methodology}). It's also possible that a matrix is nearly full rank, but has low effective rank.

\paragraph{Practical Implications for Efficiency.} Beyond its theoretical interest, the MTH offers immediate practical benefits for RLVR research. By training only 1\% of parameters, researchers can significantly reduce the memory footprint of optimizer states and gradients, enabling the finetuning of larger models on consumer-grade hardware or the use of larger batch sizes. Unlike methods like LoRA, which require learning a low-rank adapter, random sparse training utilizes the model’s existing weights directly, acting as a highly efficient, unstructured Parameter-Efficient Fine-Tuning (PEFT) baseline.

\paragraph{Redundancy and Pretraining.} Crucially, our findings highlight that this redundancy is a byproduct of the pretraining process itself. The success of random masks when training from scratch suggests that pretraining ``delocalizes'' knowledge across the parameter space, creating the very landscape that RLVR subsequently navigates.

\section{Limitations and Future Work}
\label{sec:limitations}

While the Multiple Ticket Hypothesis provides a robust framework for understanding RLVR sparsity, several limitations remain:

\begin{itemize}
    \item \textbf{Catastrophic Forgetting.} 
Catastrophic forgetting has been explored in deep learning. Specifically in generative AI, skills acquired during pretraining are lost during subsequent finetuning stages \cite{shenfeld2025rl, kirkpatrick2017overcoming, luo2025empirical, guo2025comprehensive}.

It's been established that RLVR forgets less than SFT and we tie this to Zhu et al., \yrcite{zhu2025path}'s work, as Zhu et al., \yrcite{zhu2025path} showed that naturally, RLVR chooses low principal weight directions. 

We conjecture that using random sparse mask for RLVR training is likely to lead to more catastrophic forgetting because random sampling doesn't guarantee that principal weights aren't selected and put under RLVR's optimization pressure, which is likely to lead to catastrophic forgetting. We leave further exploration of this to future work.

    \item \textbf{Task Complexity and Sparsity Thresholds:} We observed a consistent performance collapse when trainable parameters dropped below $\sim$0.01\%. While this threshold held across our reasoning tasks, more complex, multi-domain and longer horizon tasks (which is what the bulk of RLVR in practise is used for today) might require a higher ``intrinsic dimensionality'' and thus a lower maximum sparsity.
    
    \item \textbf{Model Scale:} Our experiments were conducted on models up to 1.5B parameters. While the MTH appears to hold as model size increases, further validation on frontier-scale models (e.g., 70B+) is necessary to confirm if the ratio of ``winning tickets'' remains constant or grows with scale. We conjecture however that the MTH findings will hold for larger models as increasingly larger models are more overparameterized \cite{wang2025larger} and they would exhibit more parameter redundancy.
    
    \item \textbf{Stability and Model Collapse:} We noted a higher frequency of model collapse at extreme sparsities. This suggests that while viable tickets exist at 99.9\% sparsity, the optimization path to find them becomes increasingly narrow and sensitive to hyperparameter choices like learning rate and is deserving of more attention.
\end{itemize}

\newpage




\section*{Impact Statement}
This paper presents work whose goal is to advance the field of Machine
Learning. There are many potential societal consequences of our work, none
which we feel must be specifically highlighted here.

\nocite{langley00}

\bibliography{example_paper}
\bibliographystyle{icml2026}

\newpage
\appendix
\onecolumn



\section{Complete Experimental Setup}
\label{sec:complete-exp-setup}
\subsection{Hyperparameters}
\begin{table*}[h]
  \centering
  \caption{Hyperparameters table}
  \label{tab:hyper-param_table}
  \footnotesize
  \begin{tabular}{lccc}
    \toprule
    Hyperparameter & \multicolumn{3}{c}{Training Setup} \\
    \cmidrule(lr){2-4}
    & 0.5B (Alphabet Sort) & 0.5B (GSM8k) & 1.5B (Maths) \\
    \midrule
    training steps & 150 & 100 & 100 \\
    batch size & 512 & 512 & 512\\
    num rollouts & 16 & 16 & 8\\
    max tokens & 128 & 1024 & 2048\\
    sampling temperature & 1.0 & 1.0 & 1.0 \\
    Optimizer & AdamW & AdanW & AdamW \\
    clip ratio & 0.2 & 0.2 & 0.2 \\
    KL constant & 0 & 0 & 0 \\
    weight decay & 0.01 & 0.01 & 0.01 \\
    max gradient norm  & 1.0 & 1.0 & 1.0 \\
    training seed & 42 & 42 & 42 \\
    \midrule
    eval interval & 10 & 10 & 10\\
    eval sampling temperature & 0.7 & 0.9 & 0.9 \\
    eval max tokens & 128 & 1024 & 1024\\
    eval metric & pass@1 & pass@1 & pass@1 \\
    num eval samples & 1024 & 1319 & 1319 \\
    eval seed & 2001 & - \\
    
    \midrule
  \end{tabular}

  \vskip 0.1in
\end{table*}

For alphabet sort, max number of turns and min number of turns are set to 2.

\subsection{Prompt Template}
\textbf{Mathematical reasoning}

We use the following instruction template for all training and evaluation rollouts, with only the task-specific instruction changing):

\begin{tcolorbox}[colback=gray!8!white, colframe=gray!40!black, boxrule=0.5pt, arc=3pt, left=8pt, right=8pt, top=6pt, bottom=6pt]
\begin{verbatim}
SYSTEM:
You are a helpful mathematical AI assistant. Please reason step by step 
and put your final answer within \boxed{}

USER:
[TASK_DESCRIPTION]

ASSISTANT:
\end{verbatim}
\end{tcolorbox}

\vspace{0.1in}

\noindent

\textbf{Alphabet sort}
We do not use any system prompt template for the Alphabet sort task.
The dataset itself already contains instructions.

\begin{tcolorbox}[colback=gray!8!white, colframe=gray!40!black, boxrule=0.5pt, arc=3pt, left=8pt, right=8pt, top=6pt, bottom=6pt]
\begin{verbatim}
SYSTEM:

USER:
Sort the following names in alphabetical order:
[List of names]

ASSISTANT:
\end{verbatim}
\end{tcolorbox}

\vspace{0.1in}

\subsection{Additional details on Masks training}
The Qwen models have the output projection head tied to the embedding layer; we sample masks once for the embedding layer and reuse them during projection.

\section{Eigenspectrum Analysis Methodology}
\label{sec:eigenspectrum-methodology}

To compute the eigenspectrum of the gradient Fisher information matrix (Figure~\ref{fig:eigenspectrum} in the main text), we used the following procedure:

\begin{enumerate}
    \item \textbf{Gradient collection:} Using the Qwen2.5-0.5B model on the Alphabet Sort task, we ran 150 training steps and saved the complete gradient vector at each step.
    
    \item \textbf{Gradient matrix construction:} We flattened each gradient tensor into a 1D vector of dimension $d \approx 490{,}000{,}000$ (corresponding to the total number of model parameters). Stacking all 150 gradient vectors produced a matrix $G \in \mathbb{R}^{150 \times 490M}$.
    
    \item \textbf{Gram matrix computation:} Rather than computing the full Fisher matrix $F = G^\top G \in \mathbb{R}^{490M \times 490M}$ (which would be computationally infeasible), we computed the Gram matrix $GG^\top \in \mathbb{R}^{150 \times 150}$.
    
    \item \textbf{Eigendecomposition:} We performed eigenvalue decomposition on $GG^\top$ to obtain the eigenspectrum. The non-zero eigenvalues of $GG^\top$ are identical to those of $G^\top G$, allowing us to characterize the effective rank of the gradient space.
\end{enumerate}

This procedure reveals that the gradient updates lie in a low-dimensional subspace, as evidenced by the rapid eigenvalue decay shown in Figure~\ref{fig:eigenspectrum}. The top few eigenvalues capture most of the variance, supporting Assumption~5.1 (low effective rank) in our theoretical framework.

\section{Baselines}
\label{sec:structured_baselines}
We run structured sparsity baselines experiments (first and last layer) against full parameter finetune and random seed (seed = 0).

\begin{figure*}[!th]
  \centering
  \begin{subfigure}{0.49\textwidth}
    \centering
    \includegraphics[width=\linewidth]{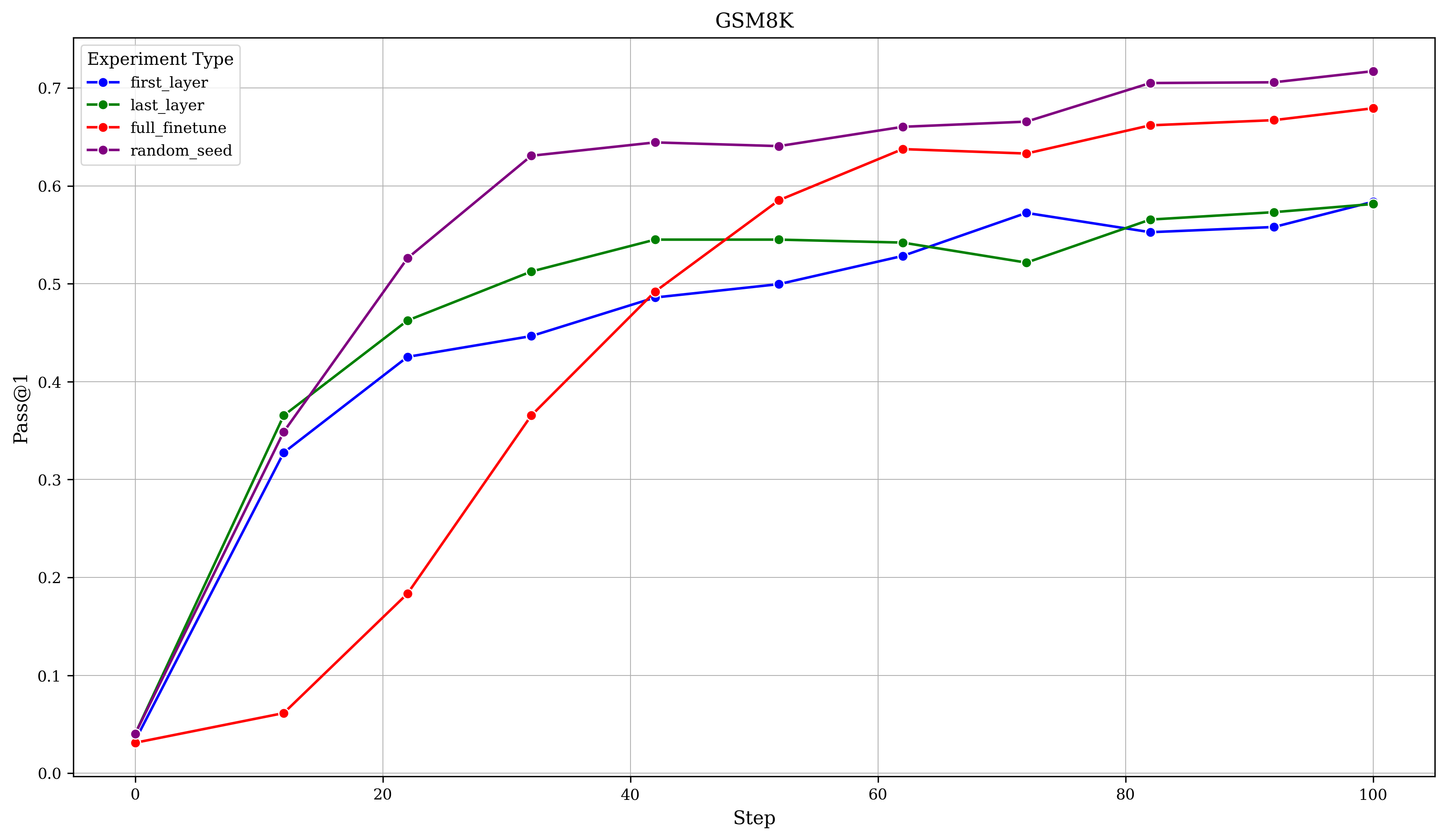}
  \end{subfigure}\hfill
  \begin{subfigure}{0.49\textwidth}
    \centering
    \includegraphics[width=\linewidth]{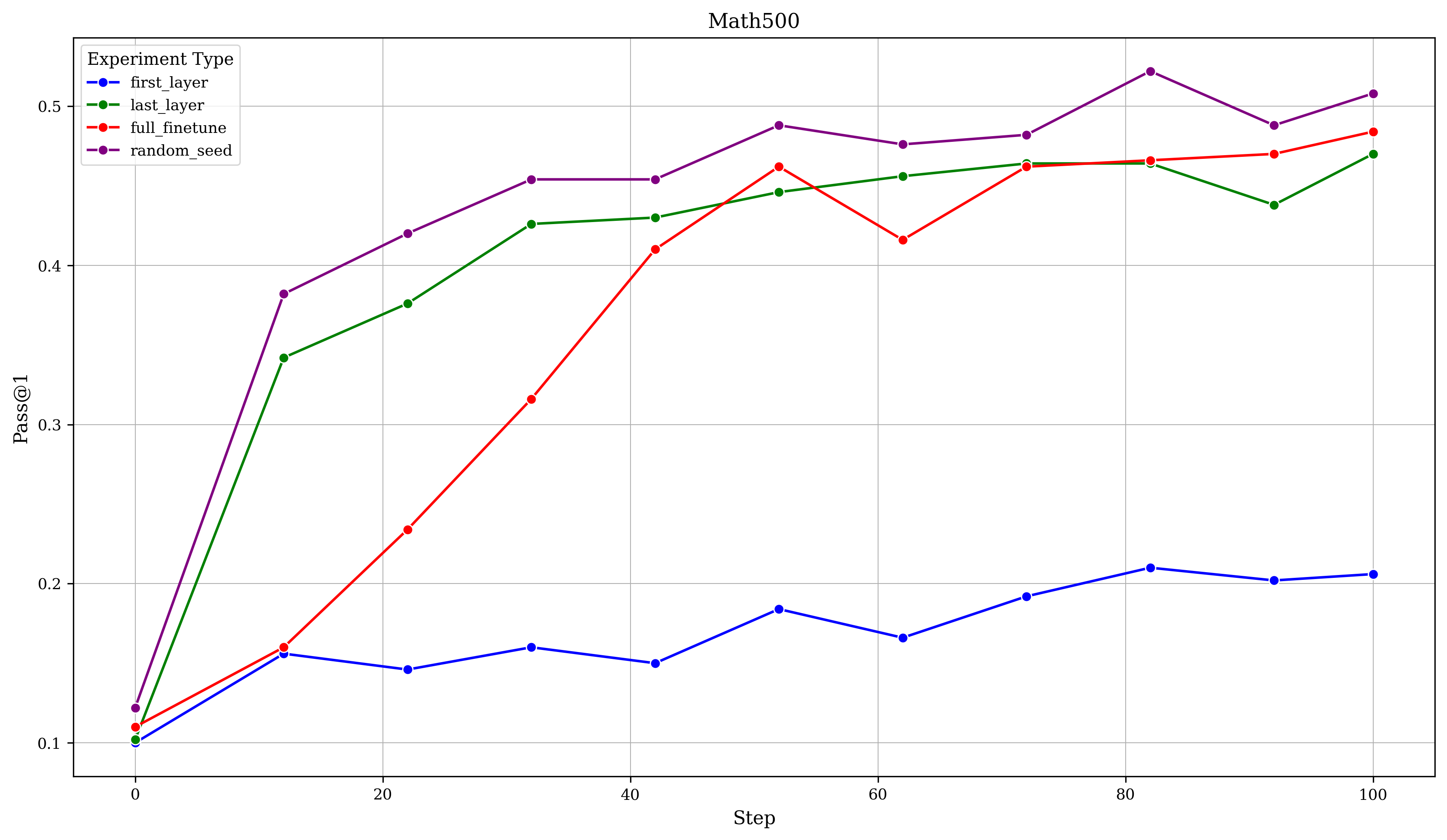}
  \end{subfigure}

  \vspace{0.8em}

  \caption{
    \textbf{Comparison of random sparse training, full parameter finetuning and structured sparsity training on Qwen-2.5-1.5B.} 
  }
  \label{fig:mth_15B_baselines}
\end{figure*}

\begin{figure*}[!h]
  \centering
  \begin{subfigure}{0.49\textwidth}
    \centering
    \includegraphics[width=\linewidth]{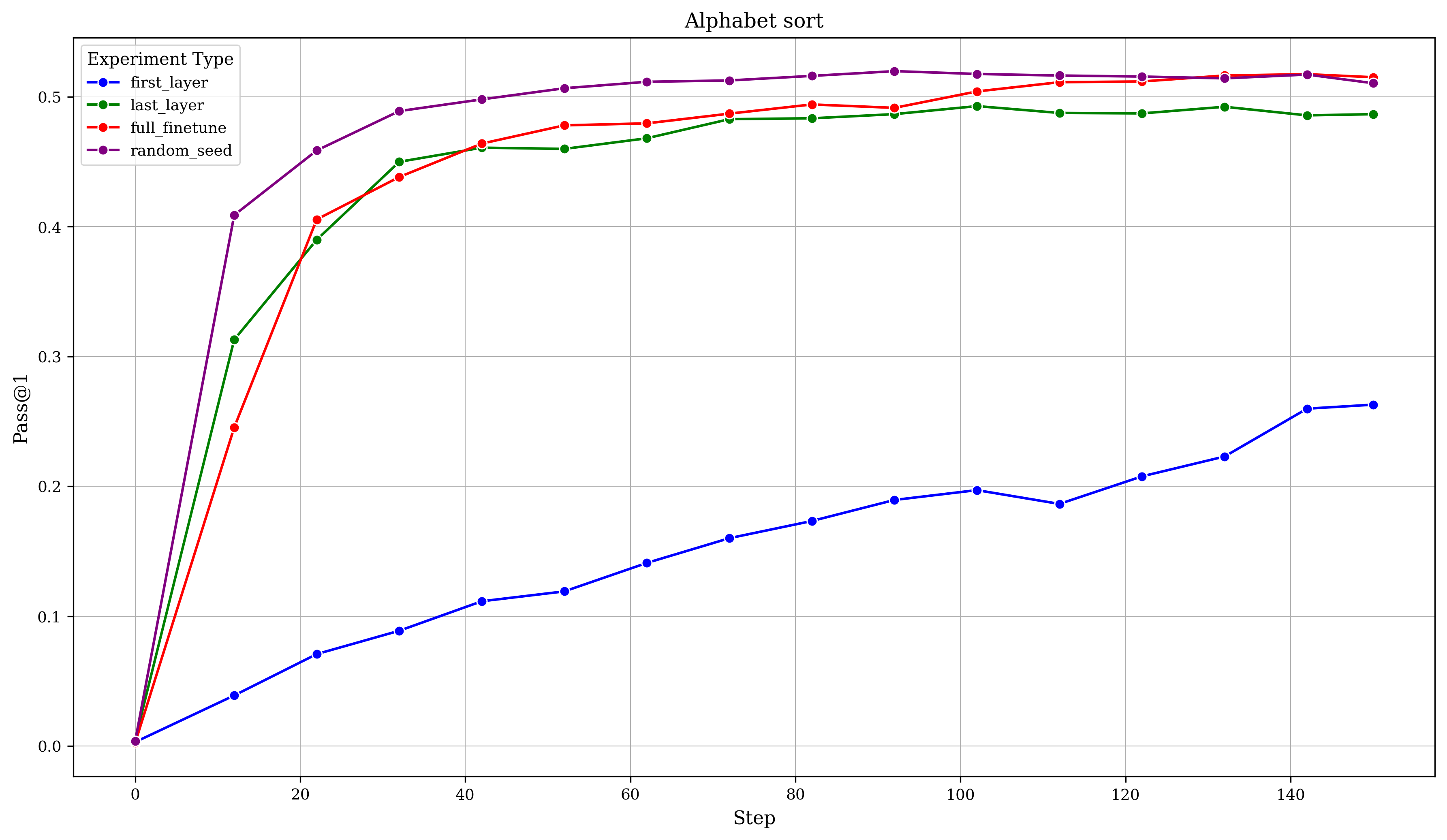}
  \end{subfigure}\hfill
  \begin{subfigure}{0.49\textwidth}
    \centering
    \includegraphics[width=\linewidth]{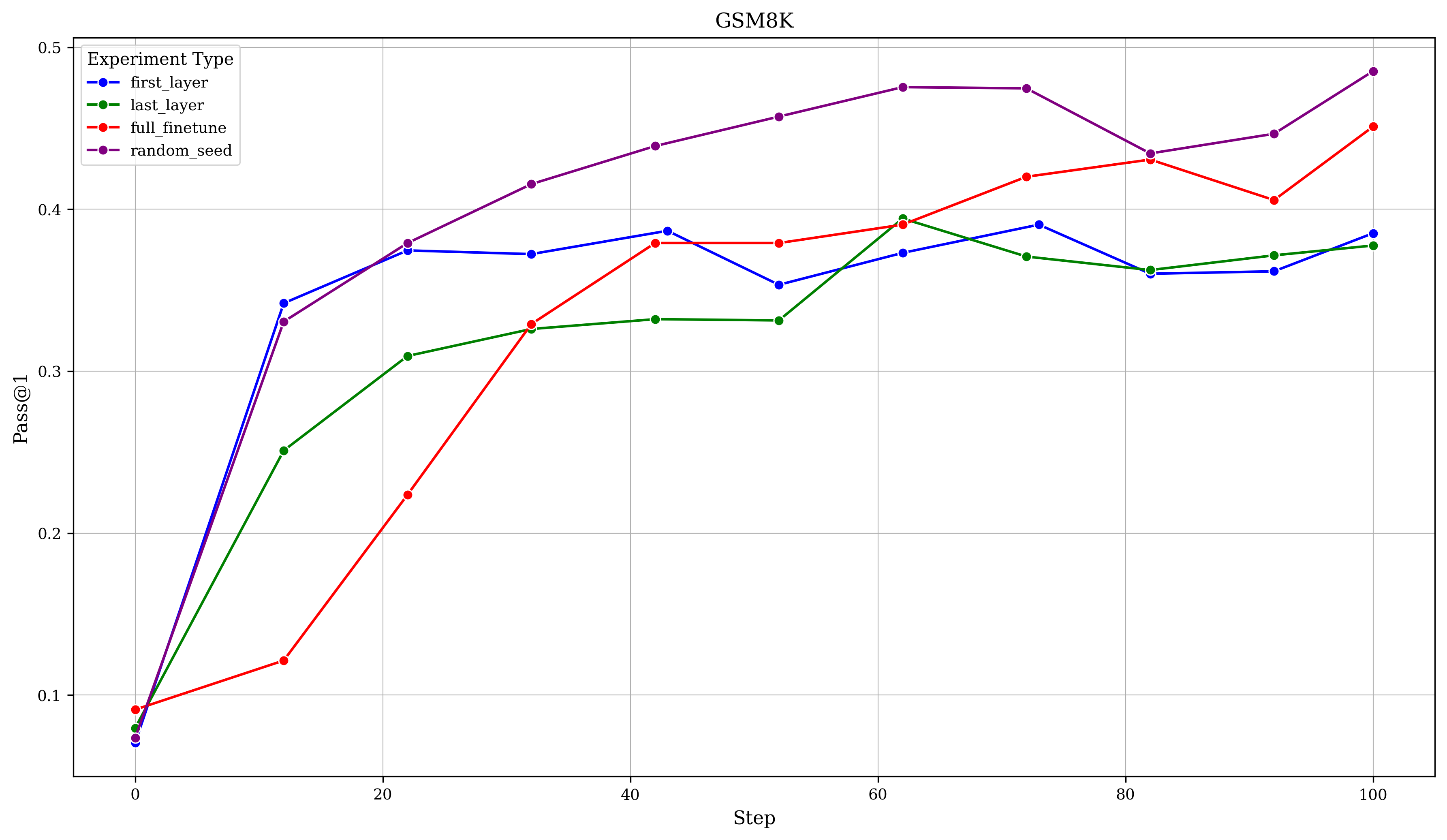}
  \end{subfigure}

  \vspace{0.8em}

  \caption{
    \textbf{Comparison of random sparse training, full parameter finetuning and structured sparsity training on Qwen-2.5-0.5B (Instruct / Base).} 
  }
  \label{fig:05B_baselines}
\end{figure*}

\newpage
\newpage

\section{Learning Rate Puzzle}
As noted by Xu \& Zhang, \yrcite{xu2024random} and evident from Table \ref{tab:sparsity_active_lr}, we also observe that at increasing sparsities, higher learning rates are needed for the performance to match full RLVR finetuning.

Figure \ref{fig:sparsity-lr-sweep-as} shows the learning rate sweep across multiple masks on Alphabet sort task, for Qwen2.5-0.5B-Instruct.

\begin{figure*}[h]
  \centering
    \includegraphics[width=\textwidth,height=0.3\textheight,keepaspectratio]
    {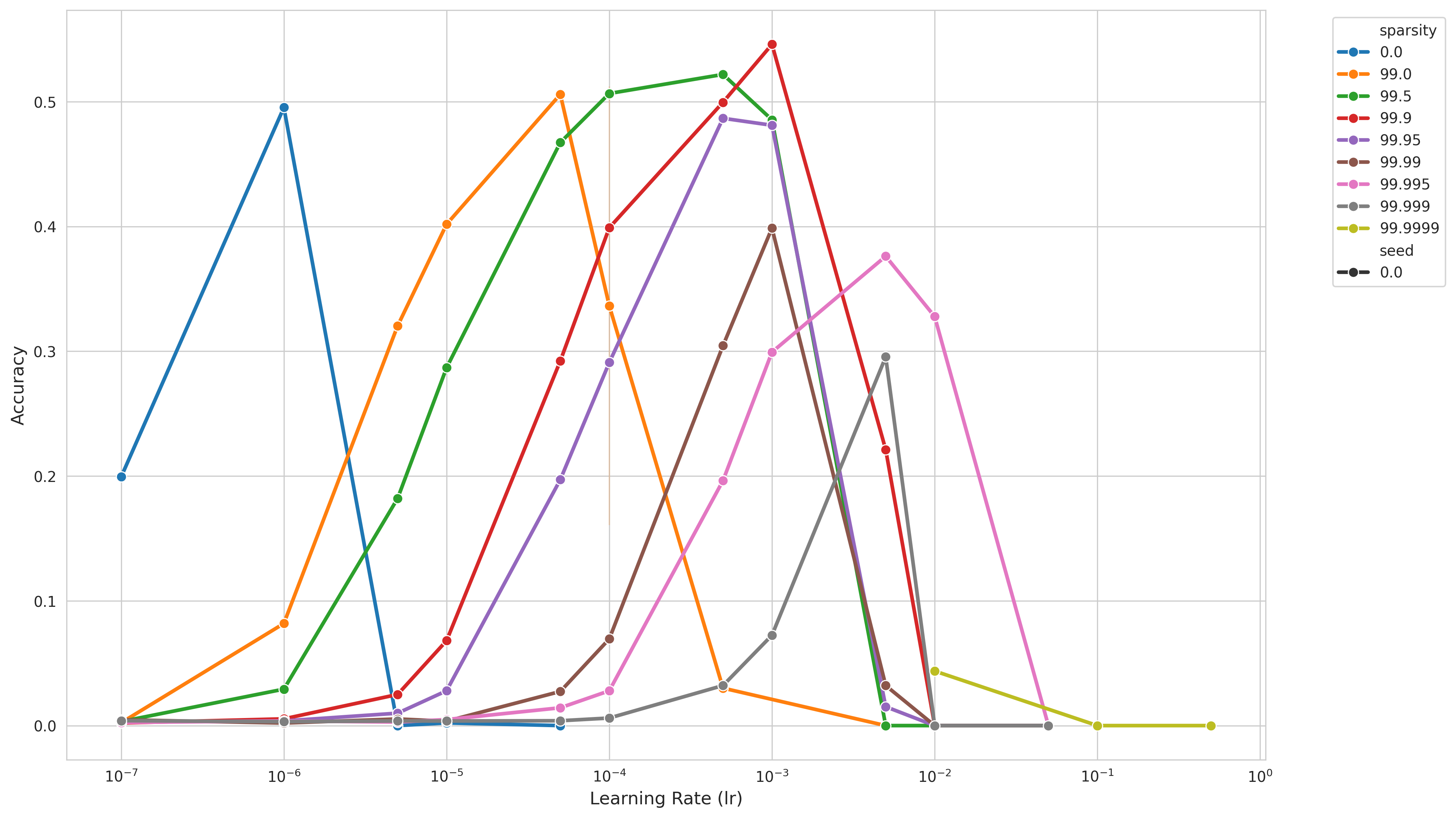}
  \caption{\textbf{Learning-rate sweep on Alphabet Sort for Qwen2.5-0.5B-Instruct.} Validation performance across sparsity levels and learning rates. All the masks for the various sparsities are seeded to 0. See the Appendix for the lr. sweep for the other seeded random masks.}
  \label{fig:sparsity-lr-sweep-as}
\end{figure*}

\section{Memory Savings Across Sparsities}
We update and track optimizer state for only the parameters being updated in a particular run. In the table below, we show the memory savings from tracking only the parameters being updated.

\begin{table}
\centering
\caption{Optimizer memory footprint during training (MiB). Full finetuning shown for reference.}
\label{tab:optimizer-memory}
\small
\begin{tabular}{lrrrrrrr}
\toprule
Model          & Full & 99\%   & 99.5\% & 99.9\% & 99.95\% & 99.99\% & 99.999\% \\
\midrule
Qwen2.5-1.5B   & 11776  & 235.8    & 118    & 23.8     & 12      & 2.5     & 0.39      \\
Qwen2.5-0.5B   & 3769.5  & 75.6     & 37.86     & 7.67    & 3.9     & 0.9     & 0.2      \\
\bottomrule
\end{tabular}
\end{table}

\newpage
\newpage 

\section{Theoretical Proofs}
\label{app:proofs}

This appendix provides complete proofs for the theoretical results stated in Section~\ref{sec:theory}.

\subsection{Detailed Assumptions and Justifications}
\label{app:assumptions}

We restate the assumptions from the main text with full justification.

\begin{assumption}[Low Effective Rank]
\label{app:ass1}
There exists a small integer $r \ll d$ and a small constant $\epsilon > 0$ such that:
\[
\frac{\sum_{i=1}^r \lambda_i}{\sum_{i=1}^d \lambda_i} \geq 1 - \epsilon.
\]
\end{assumption}

\textit{Justification.} This assumption states that the top $r$ eigenvectors capture nearly all the ``energy'' of the Fisher matrix. The empirical eigenvalue spectrum (Figure~\ref{fig:eigenspectrum}), showing rapid decay after the first few components, directly supports this assumption. Details of the eigenspectrum computation are provided in Appendix~\ref{sec:eigenspectrum-methodology}.

\begin{assumption}[Delocalization of Eigenvectors]
\label{app:ass2}
There exists a constant $\mu > 0$ such that for each eigenvector $v_i$:
\[
\|v_i\|_\infty \leq \frac{\mu}{\sqrt{d}}.
\]
\end{assumption}

\textit{Justification.} This condition states that no single parameter dominates an eigenvector; the eigenvector's mass is spread across many coordinates. This is a common property in large random matrices and is empirically plausible for well-trained neural networks, where gradient information is typically distributed across many parameters rather than concentrated in a few.

\begin{assumption}[Small-Step Regime]
\label{app:ass3}
The per-step update $\Delta$ satisfies $\|\Delta\| = O(\sqrt{K})$, where $K$ is the KL bound.
\end{assumption}

\textit{Justification.} This ensures that second-order Taylor expansions are accurate and higher-order terms are negligible. In practice, the clipping mechanism in PPO/GRPO and the on-policy sampling procedure naturally enforce small per-step policy changes.

\subsection{Proof of Proposition~\ref{prop:lowdim} (Low-Dimensional Policy Sensitivity)}
\label{app:proof1}

\begin{proposition}[Restated]
Under Assumptions~\ref{app:ass1} and~\ref{app:ass3}, for any update $\Delta$ satisfying the per-step KL constraint $D_{\mathrm{KL}}(\pi_{\theta+\Delta} \| \pi_\theta) \leq K$, the second-order change in the policy depends only on the projection of $\Delta$ onto the subspace $U = \mathrm{span}\{v_1, \ldots, v_r\}$. Components orthogonal to $U$ have negligible impact on the policy.
\end{proposition}

\begin{proof}
We proceed in five steps.

\paragraph{Step 1: KL Constraint in Quadratic Form.}
Using a second-order Taylor expansion and the definition of the Fisher matrix, the KL divergence can be approximated as:
\[
D_{\mathrm{KL}}(\pi_{\theta+\Delta} \| \pi_\theta) = \frac{1}{2} \Delta^\top F(\theta) \Delta + O(\|\Delta\|^3).
\]
Under the small-step regime (Assumption~\ref{app:ass3}), the cubic term is negligible, so the constraint is essentially:
\begin{equation}
\Delta^\top F(\theta) \Delta \leq 2K. \tag{$\star$}
\end{equation}

\paragraph{Step 2: Decomposition of the Update.}
Decompose $\Delta$ into two orthogonal components:
\[
\Delta = \Delta_\parallel + \Delta_\perp,
\]
where $\Delta_\parallel \in U$ and $\Delta_\perp \in U^\perp$ (the orthogonal complement, spanned by $v_{r+1}, \ldots, v_d$). Substituting into ($\star$):
\[
\Delta^\top F(\theta) \Delta = \Delta_\parallel^\top F(\theta) \Delta_\parallel + \Delta_\perp^\top F(\theta) \Delta_\perp \leq 2K.
\]

\paragraph{Step 3: Bounding the Contribution of $\Delta_\perp$.}
Since $\Delta_\perp$ lies in the span of the tail eigenvectors:
\[
\Delta_\perp^\top F(\theta) \Delta_\perp = \sum_{i=r+1}^d \lambda_i \langle \Delta_\perp, v_i \rangle^2 \leq \lambda_{r+1} \|\Delta_\perp\|^2,
\]
where $\lambda_{r+1}$ is the largest eigenvalue in the tail. By Assumption~\ref{app:ass1}, $\lambda_{r+1}$ is very small relative to the total variance. More precisely, let $\Lambda_{\mathrm{tail}} = \sum_{i=r+1}^d \lambda_i$. Then:
\[
\lambda_{r+1} \leq \frac{\Lambda_{\mathrm{tail}}}{d - r} \leq \frac{\epsilon \cdot \mathrm{Tr}(F(\theta))}{d - r}.
\]
Since $\mathrm{Tr}(F(\theta))$ is typically $O(d)$ in neural networks, $\lambda_{r+1} = O(\epsilon)$. Therefore, even if $\|\Delta_\perp\|^2$ is as large as $O(K / \lambda_{\min})$ (where $\lambda_{\min}$ is the smallest eigenvalue), the product $\lambda_{r+1} \|\Delta_\perp\|^2$ remains $O(\epsilon K / \lambda_{\min})$. Given that $\epsilon$ is small and $\lambda_{\min}$ is not extremely small in practice, this term is negligible compared to the KL budget $K$.

\paragraph{Step 4: Policy Change Depends Primarily on $\Delta_\parallel$.}
Consider the change in log-probability for a specific output $y$:
\[
\log \pi_{\theta+\Delta}(y|x) - \log \pi_\theta(y|x) = \Delta^\top g(y) + \frac{1}{2} \Delta^\top H(y) \Delta + O(\|\Delta\|^3),
\]
where $g(y) = \nabla_\theta \log \pi_\theta(y|x)$ and $H(y) = \nabla_\theta^2 \log \pi_\theta(y|x)$. The expected square of the linear term is $\Delta^\top F(\theta) \Delta$, which we have already bounded. The linear term decomposes as:
\[
\Delta^\top g(y) = \Delta_\parallel^\top g(y) + \Delta_\perp^\top g(y).
\]
The variance of the second term is:
\[
\mathbb{E}_{y \sim \pi_\theta} \bigl[ (\Delta_\perp^\top g(y))^2 \bigr] = \Delta_\perp^\top F(\theta) \Delta_\perp,
\]
which is negligible as argued above. Moreover, since $g(y)$ lies in the span of the Fisher eigenvectors (by definition of $F(\theta)$), the component $\Delta_\perp^\top g(y)$ is only excited by tail eigenvectors, which have small eigenvalues and hence small typical magnitudes. Therefore, the change in log-probability---and thus the policy itself---is dominated by $\Delta_\parallel$.

\paragraph{Step 5: Conclusion.}
To second order, the policy update depends only on $\Delta_\parallel$. The orthogonal component $\Delta_\perp$ neither significantly affects the KL divergence nor the policy output. This establishes that the policy-relevant subspace is effectively low-dimensional.
\end{proof}

\subsection{Proof of Proposition~\ref{prop:random} (Sufficiency of Random Masks)}
\label{app:proof2}

\begin{proposition}[Restated]
Let $S \subset \{1, \ldots, d\}$ be a random subset of indices of size $k$, chosen uniformly. Under Assumptions~\ref{app:ass1} and~\ref{app:ass2}, if $k > r$, then with high probability there exists an update $\Delta_S$ supported on $S$ (i.e., $\Delta_{S,i} = 0$ for $i \notin S$) such that:
\[
\| \Delta_\parallel - \Delta_S \|_{F} \leq \eta,
\]
where $\|\cdot\|_F$ denotes the Fisher norm $\|u\|_F = \sqrt{u^\top F(\theta) u}$, and $\eta$ is a small constant that decreases as $k$ increases. Consequently, optimizing only over parameters in $S$ can achieve policy improvement equivalent to full-parameter optimization within the KL-reachable region.
\end{proposition}

\begin{proof}
We proceed in six steps.

\paragraph{Step 1: Setup and Notation.}
Let $V_r = [v_1, \ldots, v_r] \in \mathbb{R}^{d \times r}$ be the matrix whose columns are the top $r$ eigenvectors. Any vector in $U$ can be written as $V_r c$ for some coefficient vector $c \in \mathbb{R}^r$. Let $P_S$ be the projection operator that zeros out coordinates not in $S$: $(P_S u)_i = u_i$ if $i \in S$, and $0$ otherwise.

\paragraph{Step 2: Goal.}
We want to approximate a given $\Delta_\parallel = V_r c$ by a vector $\Delta_S$ supported on $S$. Equivalently, we want to find coefficients $c' \in \mathbb{R}^r$ such that $\Delta_S = P_S(V_r c')$ is close to $\Delta_\parallel$ in Fisher norm.

\paragraph{Step 3: Delocalization and Random Masks.}
By Assumption~\ref{app:ass2}, each eigenvector $v_i$ has bounded infinity norm. This delocalization property implies that when we sample a random subset $S$ of coordinates, the restricted vectors $\tilde{v}_i = P_S(v_i)$ are likely to preserve the geometric structure of the original subspace.

Formally, consider the matrix $\tilde{V}_r = P_S(V_r) \in \mathbb{R}^{d \times r}$ (which has zeros in rows outside $S$). The product $\tilde{V}_r^\top F(\theta) \tilde{V}_r$ measures how well the restricted eigenvectors capture the Fisher metric on the subspace. Because the eigenvectors are delocalized, each row of $V_r$ has small norm. A standard concentration argument (see Lemma~\ref{lem:concentration} below) shows that with high probability,
\[
\left\| \frac{d}{k} \tilde{V}_r^\top F(\theta) \tilde{V}_r - V_r^\top F(\theta) V_r \right\|_2 \leq \delta,
\]
where $\delta$ decreases with $k$. Since $V_r^\top F(\theta) V_r = \mathrm{diag}(\lambda_1, \ldots, \lambda_r)$ is diagonal with large entries, the restricted Gram matrix is also well-conditioned when $k$ is sufficiently larger than $r$.

\paragraph{Step 4: Existence of a Good Approximation.}
Because the restricted Gram matrix is well-conditioned, the linear map $c' \mapsto P_S(V_r c')$ is injective on $\mathbb{R}^r$. Thus, for any desired $\Delta_\parallel = V_r c$, we can solve the least-squares problem:
\[
\min_{c' \in \mathbb{R}^r} \| V_r c - P_S(V_r c') \|_F.
\]
The solution satisfies:
\[
\| V_r c - P_S(V_r c') \|_F \leq \kappa \cdot \| V_r c \|_F,
\]
where $\kappa$ depends on the condition number of the restricted Gram matrix. As $k$ increases, $\kappa \to 0$. Setting $\Delta_S = P_S(V_r c')$ yields the required approximation.

\paragraph{Step 5: Connection to Policy Improvement.}
Since the policy change depends continuously on the update (as shown in Proposition~\ref{prop:lowdim}), and since the Fisher norm dominates the change in log-probabilities, a small error $\eta$ in Fisher norm translates to a small error in policy improvement. Therefore, optimizing over the mask $S$ can achieve essentially the same policy improvement as full-parameter optimization, provided $k > r$.

\paragraph{Step 6: Threshold Effect.}
The quality of approximation undergoes a phase transition: when $k < r$, the restricted Gram matrix becomes singular, and approximation fails. When $k > r$, the error decreases as $k$ increases. This explains the empirical observation that random masks work well above a certain sparsity threshold.
\end{proof}

\subsection{Technical Lemmas}
\label{app:lemmas}

\begin{lemma}[Concentration of Restricted Gram Matrix]
\label{lem:concentration}
Let $V_r \in \mathbb{R}^{d \times r}$ have orthonormal columns with $\|v_i\|_\infty \leq \mu/\sqrt{d}$. Let $S$ be a random subset of size $k$. Then with probability at least $1 - \delta$,
\[
\left\| \frac{d}{k} \tilde{V}_r^\top \tilde{V}_r - I_r \right\|_2 \leq C \mu^2 r \sqrt{\frac{\log(1/\delta)}{k}},
\]
where $\tilde{V}_r = P_S(V_r)$ and $C$ is an absolute constant.
\end{lemma}

\begin{proof}[Proof Sketch]
This follows from matrix Bernstein inequalities applied to the sum of independent random matrices $X_j = \frac{d}{k} \mathbf{1}_{j \in S} (V_r)_{j:}^\top (V_r)_{j:}$, where $(V_r)_{j:}$ is the $j$-th row of $V_r$. The delocalization assumption ensures each term has bounded norm $\|X_j\|_2 \leq \frac{d}{k} \cdot \frac{\mu^2}{d} = \frac{\mu^2}{k}$. Applying matrix Bernstein with variance proxy $\sigma^2 = O(r/k)$ yields the stated bound.
\end{proof}

\textit{Remark on the Concentration Bound.} The proof relies heavily on the delocalization assumption (Assumption~\ref{app:ass2}). While this is plausible for large neural networks, it is difficult to verify rigorously. However, empirical studies of eigenvectors in trained networks often show diffuse weight distributions, supporting this assumption. Additionally, the concentration bound requires $k = \Omega(r \log r)$, which is consistent with our empirically observed sparsity threshold.

\subsection{Synthesis: Why Random Masks Work}
\label{app:synthesis}

Propositions~\ref{prop:lowdim} and~\ref{prop:random} together explain the empirical success of random sparse fine-tuning in RLVR:

\begin{enumerate}
    \item \textbf{KL constraints create a low-dimensional trust region.} The per-step KL bound restricts updates to a region defined by the Fisher matrix's quadratic form.
    
    \item \textbf{The Fisher matrix has low effective rank.} Due to Assumption~\ref{app:ass1}, the policy-relevant subspace is only $r$-dimensional, where $r \ll d$.
    
    \item \textbf{Delocalization enables random sampling.} Due to Assumption~\ref{app:ass2}, a random mask of size $k > r$ captures this subspace with high probability.
    
    \item \textbf{Multiple masks succeed.} The combinatorial number of ways to choose $k$ parameters from $d$---each capable of spanning the same $r$-dimensional subspace---directly yields the Multiple Ticket Hypothesis.
\end{enumerate}

\subsection{Connection to Neural Tangent Kernel Theory}
\label{app:ntk}

Our theoretical framework connects to Neural Tangent Kernel (NTK) theory. In the infinite-width limit, neural networks operate in a ``lazy training'' regime where the kernel remains approximately constant. While our setting involves finite-width networks with potentially evolving representations, the low effective rank of the Fisher matrix suggests a similar phenomenon: the policy-relevant directions are determined early and remain stable, allowing arbitrary parameter subsets to navigate this low-dimensional landscape.

The delocalization assumption (Assumption~\ref{app:ass2}) is particularly natural in the NTK regime, where eigenvectors of the kernel matrix tend to be spread across many input dimensions rather than localized. This provides theoretical grounding for our empirical observation that random masks work across different model architectures and scales.

\end{document}